\documentclass[letterpaper, 10 pt, conference]{ieeeconf}  

\usepackage{wrapfig}
\usepackage{tikz}                                                          
\usepackage{graphicx}
\usepackage{amssymb}
\usepackage{epstopdf}
\usepackage{graphics}
\usepackage{amsmath, amssymb}
\numberwithin{equation}{section}
\usepackage{graphicx}   
\usepackage{verbatim}   
\usepackage{color}      
\usepackage{subfig}  
\usepackage{hyperref}   
\usepackage{xspace}
\usepackage{float}
\usepackage[font={footnotesize}]{caption}
\usepackage{csquotes}
\newtheorem{defi}{Definition}
\newtheorem{theorem}{Theorem}[section]
\newtheorem{lemma}[theorem]{Lemma}

\newtheorem{coro}[theorem]{Corollary}

\DeclareMathOperator*{\argmax}{arg\,max}
\DeclareMathOperator*{\argmin}{arg\,min}

\newcommand{\bs}[1]{\boldsymbol #1}

\newcommand{\sigp}{\textrm{sIGP}}
\newcommand{\bfx}{\mathbf{x}}

\newcommand{\bLambda}{\bs\Lambda}
\newcommand{\balpha}{\bs\alpha}
\newcommand{\bff}{\mathbf{f}}
\newcommand{\bfr}{\bff^{R}}

\newcommand{\bfh}{\mathbf{h}}

\newcommand{\bfi}{\bff^i}

\newcommand{\bfz}{\mathbf{z}}
\newcommand{\bfzt}{\mathbf{z}_{1:t}}
\newcommand{\bfzr}{\bfzt^R}
\newcommand{\bfzf}{\bfzt^{\bff}}
\newcommand{\bfzeye}{\bfzt^{\bff^i}}

\newcommand{\bmu}{\bs\mu}

\newcommand{\calX}{\mathcal X}
\newcommand{\calN}{\mathcal N}

\newcommand{\pf}{p(\bff\mid\bfzf)}
\newcommand{\probot}{p(\bfr\mid\bfzr)}

\newcommand{\peye}{p(\bfi\mid\bfzeye)}

\newcommand{\pigpshort}{p(\bfr,\bff\mid \bfzr,\bfzf)}
\newcommand{\pigpshortz}{p(\bfr,\bff\mid \bfzt)}

\newcommand{\pigpeyez}{p(\bfr,\bfi\mid \bfzt)}

\newcommand{\dieye}{\delta(\bfr,\bfi)}
\newcommand{\bdi}{\bar\delta(\bfr,\bff)}
\newcommand{\bdieye}{\bar\delta(\bfr,\bff^i)}
\newcommand{\bdieyelk}{\bar\delta(\bfr_\ell,\bff_k^i)}
\newcommand{\interaction}{\psi(\bfr,\bff)}
\newcommand{\interactioneye}{\psi(\bfr,\bff^i)}
\newcommand{\pnot}{\mathbf{P}_{\neg\kappa}}

\newcommand{\neye}{N^{\bff^i}_t}
\newcommand{\nf}{N^{\bff}_t}
\newcommand{\nr}{N^R_t}

\newcommand{\nsigp}{N_{sIGP}}
\newcommand{\nt}{n_t}

\newcommand{\bmur}{\bmu^R}
\newcommand{\bSigmar}{\bSigma^R}
\newcommand{\bmurl}{\bmu^R_\ell}
\newcommand{\bSigmarl}{\bSigma^R_\ell}

\newcommand{\bmufi}{\bmu^{\bff^i}}
\newcommand{\bSigmafi}{\bSigma^{\bff^i}}
\newcommand{\bmufik}{\bmu^{\bff^i}_k}
\newcommand{\bmufiki}{\bmu^{\bff^i}_{k_i}}
\newcommand{\bSigmafik}{\bSigma^{\bfi}_k}
\newcommand{\bSigmafiki}{\bSigma^{\bfi}_{k_i}}

\newcommand{\bSigma}{\bs\Sigma}

\newcommand{\weye}{w^{\bfi}_k}
\newcommand{\weyei}{w^{\bfi}_{k_i}}

\newcommand{\wrobot}{w^R_{\ell}}

\newcommand{\Z}{Z^{-1}_{\ell,k}}
\newcommand{\znot}{\wlk(1-\Z/\wlk)}

\newcommand{\Lambdalk}{\bLambda_{\ell,k}}

\newcommand{\wlk}{w_{\ell,k}}

\newcommand{\sumeye}{\sum_{k=1}^{\neye}}
\newcommand{\sumeyei}{\sum_{k_i=1}^{\neye}}

\newcommand{\sumfi}{\sum_{i=1}^{\nf}}
\newcommand{\sumr}{\sum_{\ell=1}^{\nr}}

\newcommand{\normalrobotshort}{\calN_{\bfr,\ell}}

\newcommand{\normalrobotl}{\calN(\bfr\mid\bmurl,\bSigmarl)}

\newcommand{\normaleyeshort}{\calN_{\bfi,k}}
\newcommand{\normaleyeshorti}{\calN_{\bfi,k_i}}

\newcommand{\normaleyek}{\calN(\bfi\mid\bmufik,\bSigmafik)}

\newcommand{\eyesum}{\sumeye\weye\normaleyek}
\newcommand{\eyesumshort}{\sumeyei\weyei\normaleyeshorti}

\newcommand{\robotsum}{\sumr\wrobot\normalrobotl}
\newcommand{\robotsumshort}{\sumr\wrobot\normalrobotshort}

\begin{document}

\title{\large{\textbf{Sparse Interacting Gaussian Processes: Efficiency and Optimality Theorems of Autonomous Crowd Navigation}}}
\author{Pete Trautman\\
Galois, Inc.}

\maketitle


\begin{abstract}
We study the sparsity and optimality properties of crowd navigation and find that existing techniques do not satisfy both criteria simultaneously: either they achieve optimality with a prohibitive number of samples or tractability assumptions make them fragile to catastrophe.  For example, if the human and robot are modeled independently, then tractability is attained but the planner is prone to overcautious or overaggressive behavior.  For sampling based motion planning of joint human-robot cost functions, for $\nt$ agents and $T$ step lookahead, $\mathcal O(2^{2\nt T})$ samples are needed for coverage of the action space.   Advanced approaches statically partition the action space into free-space and then sample in those convex regions.  However, if the human is \emph{moving} into free-space, then the partition is misleading and sampling is unsafe: free space will soon be occupied.  We diagnose the cause of these deficiencies---optimization happens over \emph{trajectory} space---and propose a novel solution: optimize over trajectory \emph{distribution} space by using a Gaussian process (GP) basis.  We exploit the ``kernel trick'' of GPs, where a continuum of trajectories are captured with a mean and covariance function.  By using the mean and covariance as proxies for a trajectory family we reason about collective trajectory behavior without resorting to sampling. The GP basis is sparse and optimal with respect to collision avoidance and robot and crowd intention and flexibility.  GP sparsity leans heavily on the insight that joint action space decomposes into free regions; however, the decomposition contains feasible solutions only if the partition is dynamically generated. We call our approach \emph{$\mathcal O(2^{\nt})$-sparse interacting Gaussian processes}.
\end{abstract}

\vspace{-5pt}
\section{Introduction}
\noindent Roboticists have been investigating navigation in human environments since the 1990s~\cite{crowd-nav-survey}.  A hallmark study was the RHINO experiment, deployed as a tour guide in Bonn, Germany~\cite{rhino-robot}.  This experiment was followed by the MINERVA robot~\cite{minervapaper}, exhibited at the Smithsonian in Washington D.C.  These studies inspired research in the broad area of robotic navigation in the presence of humans, ranging from additional tour guide work~\cite{tumcityexplorer,shiomi_rss} to field trials for social partner robots~\cite{kandasidebyside, tum-human-robot-navigation}.  The underlying navigation architecture of nearly every pre-2010 study was \emph{predict-then-act} (an unfortunate misnomer, since robot action influences human response), where human and robot motion were assumed independent.  In~\cite{trautmaniros}, the independence assumption of predict-then-act was shown to lead to robot-human miscalibration: either the planner moved overcautiously (underestimating cooperation) or overaggressively (overestimating cooperation). Since a well behaved path existed (the humans did fine), this behavior was safety or efficiency suboptimal. Importantly, failures occurred across the crowd density spectrum:~\cite{trautmanicra2013} empirically demonstrated a 3x safety decrement compared to coupled human-robot models for densities ranging from 0.1-0.7 people/m$^2$.  Additionally, work in cognitive systems engineering~\cite{woods-maba-maba,woods-joint-cog-systems,dekker-folk-models}, team cognition~\cite{cooke-interactive-team,fiore-augment-team-cognition} and human-agent-robot teamwork~\cite{bradshaw-hart,bradshaw-coordination-hart} corroborates the criticality of \emph{joint} modeling.
\begin{figure}[t]
\vspace{0pt} 
\centering 
\includegraphics[width=0.44\textwidth]{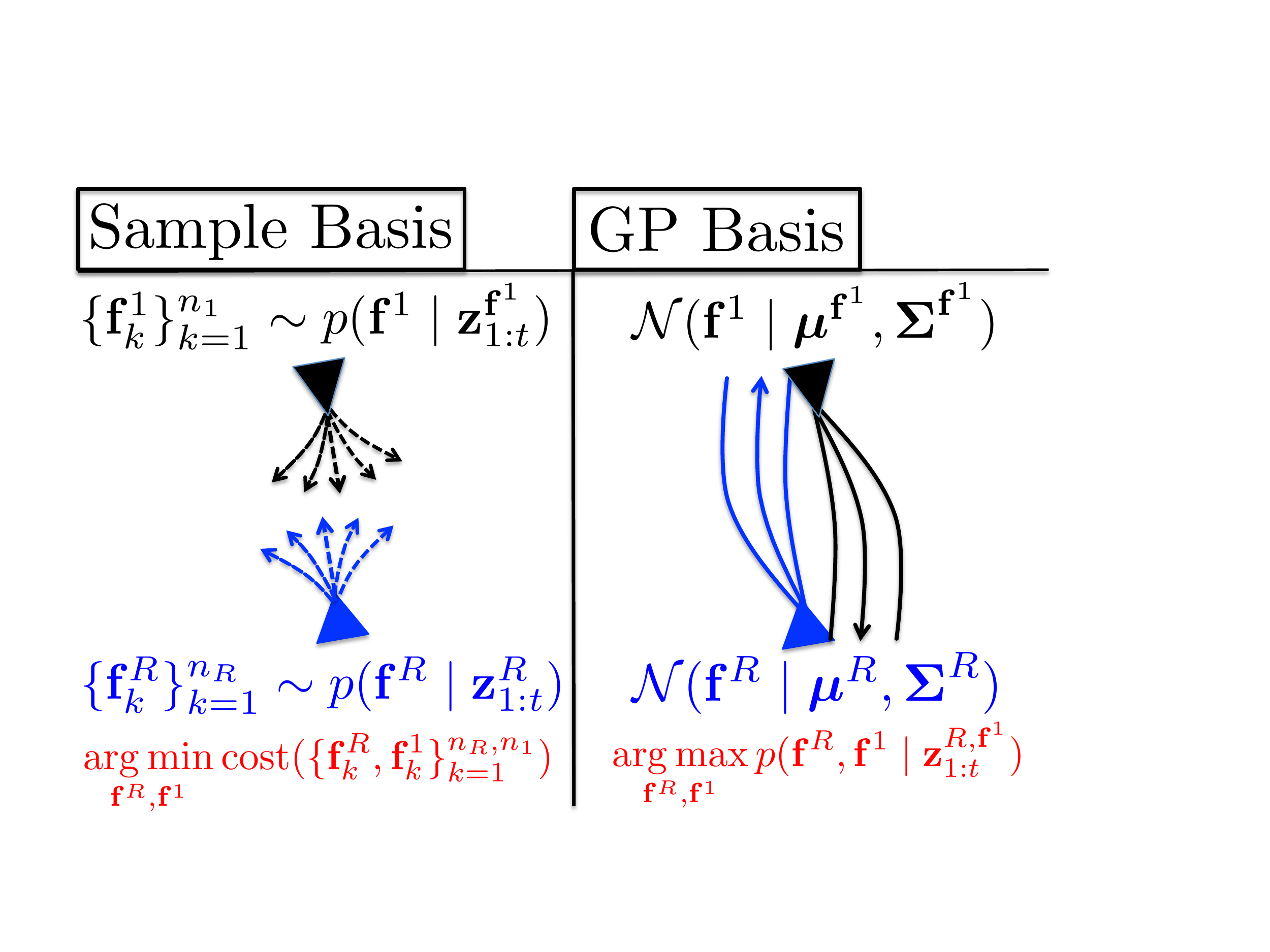}
\vspace{-5pt} 
\caption{\textbf{L} \emph{Trajectory} basis complexity is $\mathcal O(2^{2\nt T})$.  \textbf{R} \emph{Gaussian process} (GP) basis is complexity $\mathcal O(2^{\nt})$ and 2 bases capture optimal behavior.}
\label{fig:samples-vs-gps} 
\vspace{-20pt}
\end{figure}
Although~\cite{trautmanicra2013,kuderer-ijrr-2016} utilized joint human-robot models (like~\cite{helbing1}), neither achieved both optimality and tractability;~\cite{trautmanicra2013} showed that infinite samples achieves optimality (like~\cite{rrt*}), but finite sampling cannot provide guarantees.  In~\cite{kuderer-ijrr-2016}, the action space was decomposed into convex sampling regions;  as we will show, this approach fractures the pure sampling optimality guarantees.

In this paper, we extend~\cite{trautmanthesis,kuderer-ijrr-2016} so both optimality and sparsity are achieved.  We argue that correctly formulated sparsity is a precursor to optimality.  Our approach reconsiders a motion primitive---agents go left or right around each other---as a \emph{dynamic} interplay between agent Gaussian processes (GPs).   Consider a standard approach for determining free space: in Figure~\ref{fig:collide} the Voronoi diagram is built around the cyan agent at time $t$ and equal probability is assigned to left or right directions around the cyan agent.  The prediction densities in the left pane show that the cyan agent prefers left, so left and right are not equal probability (Figure~\ref{fig:bend} confirms this assertion).  Voronoi diagrams make an independence assumption: by constructing free space without considering human-robot interplay, the human and robot are decoupled (thus inducing human-robot miscalibration).  With interacting GPs, free space is dynamically generated by the co-evolution of human and robot trajectory \emph{distributions}; reactive obstacles---humans---are \emph{constraints} on the robot's motion.  Whereas interaction almost universally precipitates exponential complexity, interacting GPs \emph{decrease} joint action space size (we do not refute motion planning's complexity~\cite{piano-mover} since reactive obstacles allow higher granularity reasoning, mitigating combinatorics).  \emph{Succinctly: navigation in a GP basis is sparse and the coefficients rank optimality.}

\section{Related Work}
\label{sec:related-work}
\vspace{-0pt}
\noindent Independent agent Kalman filters are a common starting point for crowd prediction.  However, this approach leads to an uncertainty explosion that makes efficient navigation impossible~\cite{trautmaniros}. Some research has thus focused on \emph{controlling} uncertainty.  For instance~\cite{navintent,learningpeople,avoidcars} develop high fidelity independent models, in the hope that controlled uncertainty will improve navigation performance.  In~\cite{dutoit-tro} predictive uncertainty is bounded; intractability is avoided with receding horizon control~\cite{rhc-survey}; collision checking algorithms developed in~\cite{dutoit-collision-checking}, with roots in~\cite{blackmore-planning, blackmore-probabilistic-checking}, keeps navigation safe.  The insight is that since replanning is used, predictive covariance can be held constant at measurement noise.  In~\cite{georges_auro, joseph-dps-over-gps}  more sophisticated individual models are developed: GP mixtures~\cite{gpmlras} with a Dirichlet Process (DP) prior over mixture weights~\cite{teh-dp}.  DP priors allow the number of motion patterns to be data driven while the GP enables inter-motion variability; RRT~\cite{rrt} finds feasible paths; collision checking~\cite{aoude-collision-check} guarantees safety.  In the work above, independent agent models are the core contribution; we show that this is insufficient for crowd navigation.

Proxemics~\cite{hall-1966} studies proximity relationships, providing insight about social robot design: in~\cite{mead-proxemic, mead-social,takayama-proxemics} proxemics informs navigation protocol. Similarly,~\cite{kanda-position-prediction} studies pedestrian crossing behaviors using proxemics.  In~\cite{svenstrup-human} RRTs are combined with a proxemic potential field~\cite{potential-field-limits}.  Instead of using proxemic rules,~\cite{risk-rrt} adopts the criteria of~\cite{harmonious-hri}.  Personal space rules guide the robot's behavior by extending the Risk-RRT algorithm developed in~\cite{fulgenzi-motion} (Risk-RRT extends RRT  to include the probability of collision along candidate trajectories). The work in~\cite{althoff_robot_nav} is more agnostic about cultural considerations; a ``probabilistic collision cost'' is based on \emph{inevitable collision states}~\cite{probabilistic-ics}.  The work in~\cite{fraichard-three-rules} argues that robot and environment dynamics and a sufficient look-ahead guarantees collision avoidance.  Although these approaches model human-robot interaction, they do not model human-robot \emph{cooperation}: respecting a proper distance between human and robot (similar to~\cite{ziebartppp}) is emphasized.  

\emph{Human intention aware} path planning has recently become popular (see~\cite{crowd-nav-survey} for a comprehensive accounting).  In~\cite{intention-aware-crowd,intention-aware-momdp} multi-faceted human intent is modeled; the challenge is accurately inferring the \emph{true} intention and hedging against uncertainty.  This is addressed through the use of \emph{partially observable Markov decision processes} and \emph{mixed observability Markov decision processes}.   In~\cite{vaibhav-co-nav}, anticipatory indicators of human walking motion informs co-navigation;~\cite{fast-target-prediction} caches a library of motion level indicators to optimize intent classification;~\cite{disney-robot} takes a similar approach for anticipating when a human wants to interact with a robot.  In the important legibility and predictability studies of~\cite{dragan-rss-2013,dragan-hri-2013}, robot motion is optimized to meet predefined human acceptability criteria.   All these approaches model human intent \emph{a-priori}, rather than as an online human-robot interplay.  

Some approaches learn navigation strategies by observing example trajectories.  In~\cite{learning-social-robot}, learned motion prototypes  guide navigation. In~\cite{ziebartcabbie}, maximum entropy inverse reinforcement learning (max-Ent IRL) learns taxi cab driver navigation strategies from data; this method is extended to an office navigation robot in~\cite{ziebartppp}.  In~\cite{irlnavigate}, max-Ent IRL is extended to dynamic environments.  Their planner is trained using \emph{simulated} trajectories, and the method recovers a planner which duplicates the behavior of the simulator.  In~\cite{waugh-irl}, agents \emph{learn} how to act in multi-agent settings using game theory and the principle of maximum entropy.  
\begin{figure}[t]
\vspace{-0pt}
\centering
\subfloat[GPs on a collision course]{
  \includegraphics[width=0.485\linewidth]{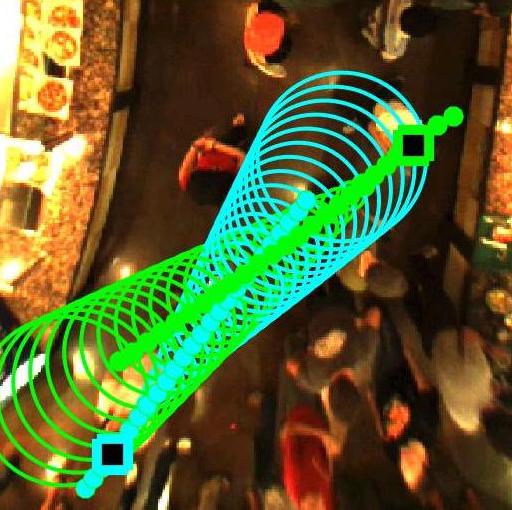}
  \label{fig:collide}
  }
\subfloat[GPs bending around each other]{
  \includegraphics[width=0.48\linewidth]{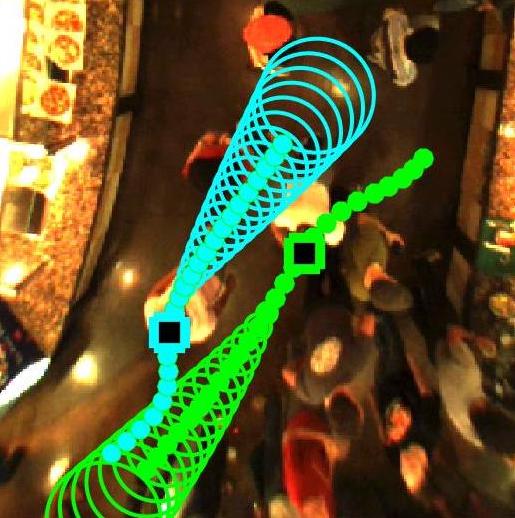}
  \label{fig:bend}
  }
  \vspace{-5pt}
\caption{Starting with GPs on a collision course, the GPs interact to ``bend'' around each other; (a, b) shows how humans actually move past one another. }
\label{fig:bend-collide}
\vspace{-15pt}
\end{figure}

In~\cite{sadigh-rss-2016} coupled human-robot models generated autonomous behaviors that were efficient and communicative (between a single human and a single robot); it remains unclear if coupled dynamical systems using reinforcement learning scales to multiple agents.  In~\cite{sadigh-iros-2016}, human state information was gathered by coupling robot action to human prediction.  Using deep learning for the crowd \emph{prediction} problem~\cite{social-lstm} raises an important question (since \cite{stanford-crowd-data} collects millions of training examples): can \emph{social navigation} be learned?  The combinatorics of social navigation (Section~\ref{sec:complexity}) makes naive approaches (e.g., brute force learning without exploiting sparsity) seem infeasible.
\begin{figure}[b]
\vspace{-20pt} 
\centering 
\includegraphics[width=0.49\textwidth]{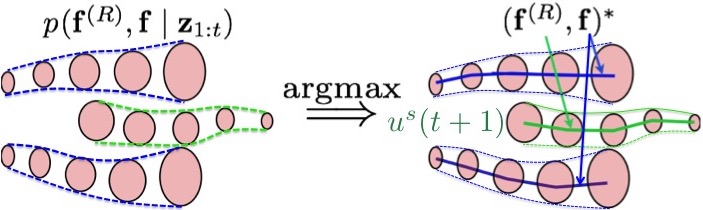}
\vspace{-15pt} 
\caption{Robot, crowd \emph{distributions} bend around each other;  $\pigpshortz$ is sparse in GP space; high value elements correspond to maxima.}
\label{fig:bending-tubes} 
\vspace{-3pt}
\end{figure}

In~\cite{kuderer-ijrr-2016} pure sampling is observed to be ineffective for coupled models; some mechanism is required to guide sampling (\cite{rvo,snape-hrvo} makes a similar observation, but RVO/HRVO was shown brittle to noise and motion ambiguity in~\cite{trautmanthesis}).  To guide sampling, Voronoi techniques applied to static obstacles parses the action space into convex regions.  However, static Voronoi techniques (or any static convex region identifier) leads to suboptimal strategies: by ignoring motion, a human-robot decoupling assumption is made.   In reactive environments, free space is dynamically generated by the probabilistic interplay of human and robot trajectory distributions (see Figure~\ref{fig:bending-tubes}).  Sparse IGP ($\sigp$) achieves this high level interaction, establishing sparsity and optimality guarantees in the process.

\vspace{-5pt}
\section{Terminology}
\label{sec:terminology}
\noindent We collect measurements $\bfz^R_{1:t}, \bfz^1_{1:t},\ldots,\bfz^{\nt}_{1:t}$ of the robot trajectory $\bfr \colon t\in\mathbb R \to \calX$, where $\calX$ is the joint action space and $\nt$ human trajectories $\bff = [\bff^1,\ldots,\bff^{\nt}]\colon t\in\mathbb R \to \calX$, which are governed by $\probot$ and $p(\bff^1\mid\bfz^{\bff^1}_{1:t}),\ldots, p(\bff^{\nt}\mid\bfz^{\bff^{\nt}}_{1:t})$.  We do not assume that every measurement of $\bfz^R_{1:t}$ and $\bfzf$ in $1:t$ is present.  We use the shorthand $\bfzf = [\bfz^{\bff^1}_{1:t},\ldots,\bfz^{\bff^{\nt}}_{1:t}]$; similarly, we let $\pf = \prod_{i=1}^{\nt}\peye.$  We model both $\bfi$ and $\bfr$ as stochastic processes.  Our robot and pedestrian models are represented in a GP basis (Figure~\ref{fig:gp-representation}):
\begin{align}
\label{eq:gp-mixtures}
\probot &= \robotsum, \nonumber\\
\peye &= \eyesum.
\end{align}   
The mixture weights are the likelihood of the data: $\wrobot=\calN(\bfr=\bfzr \mid \bmur_\ell,\bSigma_\ell)$ and $\weye= \calN(\bfi=\bfzeye \mid \bmufi_k,\bSigmafi_k)$.  Although the GPs evolve at each time step, we suppress time in the mean and covariance functions: $\bmu \equiv \bmu(t)$ and $\bSigma\equiv \bSigma(t)$.    As an illustration of the GP basis, consider Figure~\ref{fig:gp-representation} and the distribution $\peye=\sum_{k=L,R} \weye \normaleyek$.  If $w^{\bff^i}_L > w^{\bff^i}_R$ then $\argmax_{\bfh}\peye = \bmufi_L.$  Optimization over $\bmufi_L$ can be misleading since $\bmufi_R$ is ignored.  What if the human is debating whether to visit Lenny or Rhonda?  What if the human is ``flexible'' in how they intend to travel to Lenny or Rhonda?  For crowd navigation to be successful, we must reason over ambiguity and flexibility.  We make this precise.
\begin{defi}
\label{def:ambiguity}
We call $\bmufi_k,\bmu^R_j$ human and robot \emph{\textbf{intentions}}.  If $\neye,\nr>1$, \emph{\textbf{intention ambiguity}} is present.  We measure \emph{\textbf{intention preference}} with the weights $\weye,\wrobot$.  E.g, if one weight is large, intention ambiguity is small. 
\end{defi}
\begin{defi}
\label{def:flexibility}
Flexibility is the willingness of an agent to compromise about intention $\bmufi_k$ or $\bmu^R_j$.  Mathematically, the \emph{\textbf{flexibility}} of intent $\bmu^R_j$ or $\bmufi_k$ is $\bSigma^R_j$ or $\bSigmafi_k$.
\end{defi}
\noindent Flexibility is motivated by the following: suppose an agent is unimodal with model $\calN(\bfx\mid\bmu,\bSigma)$.  If the agent intends $\bmu$ strongly (by providing data supporting $\bmu$) then $\bSigma$ is small; the agent is $\bSigma$ unwilling to compromise on $\bmu$.  If the agent has not provided a strong signal supporting $\bmu$ then $\bSigma$ is small.
\begin{defi}
\label{defi:overlap}
The \emph{\textbf{probability of collision}} of $\normalrobotl$ and $\normaleyek$ is (see Section 5.2 of~\cite{trautmanthesis})
\begin{align}
\label{eq:z}
&P(\kappa)= \int \calN(\bfx \mid \bmur_\ell,\bSigmar_\ell)\calN(\bfx \mid \bmufi_k,\bSigmafi_k)d\bfx \nonumber\\
&= \wlk\exp\Big[-\frac{1}{2}(\bmur_\ell-\bmufi_k)^\top (\bSigmar_\ell + \bSigmafi_k)^{-1}(\bmur_\ell-\bmufi_k) \Big]\nonumber\\
&=\Z,
\end{align} 
where $\wlk = (2\pi)^{-T/2}|  \bSigmar_\ell +\bSigmafi_k |^{-1/2}$ and $\Z$ is the normalization coefficient resulting from multiplying two Gaussians. We note $\max \big[P(\kappa)\big] = \wlk$.   
\end{defi}
\noindent In crowd navigation we are interested in the probability of \emph{not} colliding $P(\neg\kappa) = \wlk-P(\kappa) = \znot.$  To mitigate symbol proliferation, we introduce a shorthand.
\begin{defi} Let  $\calN_{\bff^\chi,\sigma} = \calN(\bff^\chi\mid\bmu^\chi_\sigma,\bSigma^\chi_\sigma)$; $\chi \in\{R, 1,\ldots, \nt \}$, $\sigma \in \{\ell,k \}$. Thus $\calN_{\bff^3,k}=\calN(\bff^3\mid\bmu^{\bff^3}_{k},\bSigma^{\bff^3}_{k})$.
\end{defi} 
\begin{defi}The operator $\pnot$ is defined as
\label{defi:pnot}
\begin{align*}
\pnot\colon&\Big[\normalrobotshort\normaleyeshort\Big]\to \znot\normalrobotshort\normaleyeshort.
\end{align*} 
We define $\Lambdalk^{R,\bfi}\equiv\znot=P(\neg \kappa)$.
\end{defi}

\begin{figure}[t]
\vspace{5pt} 
\centering 
\includegraphics[width=0.42\textwidth]{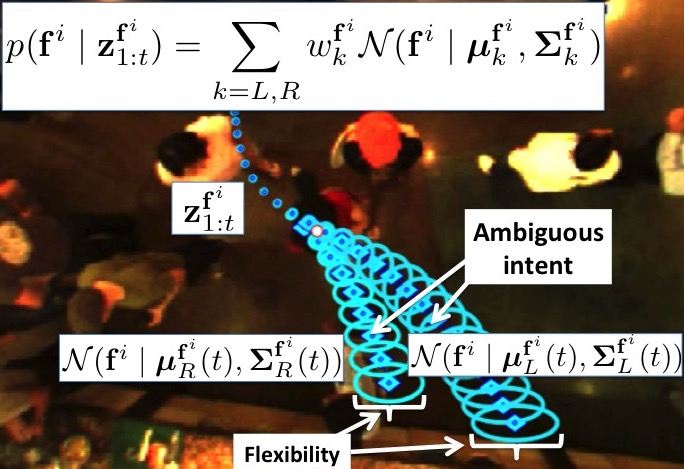}
\vspace{-5pt} 
\caption{Human ambiguity: $\calN(\bff^i\mid\bmufi_R,\bSigmafi_R),\calN(\bff^i\mid\bmufi_L,\bSigmafi_L)$ model intention to go direction $\bmufi_R,\bmufi_L$ (with preferences $w^{\bff^i}_R$ and $w^{\bff^i}_L$).  Within a mode, human has flexibility $\bSigmafi_{R,L}$ about $\bmufi_{R,L}$.  Data from~\cite{trautmanicra2013}}
\label{fig:gp-representation} 
\vspace{-15pt}
\end{figure}

\begin{defi}
\label{defi:decoupled-planning}
\emph{\textbf{Independent agent planning}} optimizes a decoupled cost function $C(\bfr,\bff) = C_R(\bfr)C_{\bff}(\bff)$.
\end{defi}

\begin{defi}
\label{defi:sbmp}
\emph{\textbf{Sampling based motion planning (SBMP)}} draws  $\bfr_k\sim\probot$ and $\bff_k\sim \pf$, and then computes the optimal joint trajectory $\argmin_{\bfr,\bff} C(\{\bfr_k,\bff_k \}_{k=1}^N)$, where $C(\{\bfr_k,\bff_k \}_{k=1}^N)$ is some joint cost function.  If we can sample uniformly from the cost function, then SBMP is a sampling based approximation of the joint distribution
\begin{align*}
\pigpshort = \sum_{k=1}^N \delta([\bfr,\bff]-[\bfr,\bff]_k).
\end{align*}
\end{defi}

\begin{defi}[\textbf{Convex lane approach~\cite{kuderer-ijrr-2016}}]
\label{defi:convex-lane}
\textbf{Convex lanes} are regions through the crowd with a single optima.  The \textbf{convex lane approach} creates a convex lane $\{\psi_i\}_{i=1}^m\in\calX$ partition, with weights $w^{\psi_i}=P(\psi_i)$, based on current pedestrian positions.  Inference samples lanes $\psi_j\sim\{w^{\psi_i}\}_{i=1}^m$ and then trajectories $[\bfr,\bff]_k\sim p_{\psi_j}(\bfr,\bff)$:
\begin{align*}
\pigpshort  &= \sum_{\psi_i\in\calX}w^{\psi_i}p_{\psi_i}(\bfr,\bff\mid\bfz^R_{1:t},\bfz^f_{1:t})\\
&=\sum_{\psi_i\in\calX}\sum_{k\in\psi_i} w^{\psi_i} \delta([\bfr,\bff]-[\bfr,\bff]_k).
\end{align*}
\end{defi}

\section{Statistical Principles of Crowd Navigation}
\label{sec:prelims}
\noindent In~\cite{trautmaniros}, robot navigation in dense human crowds was explored as a probabilistic inference problem, rather than as a cost minimization problem.  This enabled the perspective that navigation in crowds is a \emph{joint decision making problem}: how should the robot move, in concert with the humans around it, so that the intent and flexibility of each participant is simultaneously optimized?  The high level mathematics of this approach makes explicit how crowd navigation is best understood as joint decision making.  First, the joint predictive distribution $\pigpshort$ over the robot trajectory $\bfr$ and the crowd trajectory $\bff$ is formulated.  The robot's next action $u^I(t+1)$---what the robot is \emph{predicted} to do according to the human and robot models---is then clear (Figure~\ref{fig:bending-tubes}):
\begin{align}
\label{eq:igp}
[\bfr,\bff]^* &= \argmax_{\bfr,\bff}\pigpshort \\
u^I(t+1) &= \bff^{R*}(t+1)\nonumber .
\end{align}
\noindent The robot action $u^I(t+1)$  is interactive: $\argmax_{\bfr,\bff}$ balances $\bfr,\bff$'s intentions and flexibilities against collision avoidance, and so is the \emph{optimal robot-crowd decision}.  Furthermore, 
\begin{align*}
\pigpshort &= p(\bfr \mid \bfz^R_{1:t},\bff)\pf.
\end{align*}
If we have individual robot and crowd models $\probot$ and $\pf$, then a standard factorization~\cite{blake-book} is
\begin{align}
\label{eq:decompose}
\pigpshort = \interaction \probot \pf,
\end{align}
that is, $p(\bfr \mid \bfz^R_{1:t},\bff) = \interaction\probot.$
In this section we derive what property $\interaction$ \emph{must} have in order to preserve the statistical integrity of $\probot$ and $\pf$.  Using $\interaction$, we expand $\pigpshort$ in a GP basis and derive an  inference approach to find the optimal solution.  

\subsection{Statistical invariants of cooperative navigation}
\noindent For clarity, we study a single robot $\bfr$ and a single human $\bff^i$: 
\begin{align}
\label{eq:decomposeeye}
\pigpeyez&= \interactioneye p(\bfr \mid \bfz^R_{1:t},\bff^i)\peye,
\end{align}
where $\bfzt = [\bfzr,\bfzeye]$.  What should the interaction function $\interactioneye$ be?  First, let $\bfr-\bff^i \equiv [\bfr(1),\ldots,\bfr(T)]-[\bff^i(1)\ldots,\bff^i(T)]$, where $\bfr(t),\bff^i(t)\in \calX$ and we discretize the functions $\bfr,\bff^i$ by $1:T$ to define subtraction. A plausible choice is $\interactioneye = \prod_{\tau=1}^T(1-\exp(-\frac{1}{2\gamma}(\bfr(\tau)-\bff^i(\tau))^2))$ since this function models joint collision avoidance~\cite{trautmaniros}. However, recall that $\probot,\peye$ encode \emph{online} intention and flexibility information (Equation~\ref{eq:gp-mixtures}, Definitions~\ref{def:ambiguity},~\ref{def:flexibility}): the means and covariances capture inter-agent intention and flexibility that is \emph{specific to and influenced by} the environment.  If the avoidance distance $\gamma>0$ then $\interactioneye$ alters joint flexibility \emph{in a static and generic way} (Figure~\ref{fig:interaction-bar-delta}): $\interactioneye$ ignores the peculiarities of agent-specific flexibilities (e.g., agent 1 and 2 are flexible with each other in a specific way).   To preserve the statistics of $\probot,\peye$ we introduce the following ($\bfr_\ell\sim\probot,\bfi_k\sim\peye$):
\begin{align}
\label{eq:bdilk}
&\bdieyelk\equiv 
\lim_{\gamma\to 0}\Big[\prod_{\tau=1}^T\Big(1-\exp[-\frac{1}{2\gamma}\big(\bfr_\ell(\tau)-\bfi_k(\tau)\big)^2]\Big)\Big] \nonumber\\&=
\begin{cases}
1 &\text{if $\nexists t\in[1,T]$ such that }  \bfr_\ell(t)=\bfi_k(t)\\
0 &\text{if $\exists t\in[1,T]$ such that } \bfr_\ell(t) =\bfi_k(t).
\end{cases}
\end{align}
That is, $\bdieyelk$ accomplishes the following: if $\bfr_\ell(t)$ and $\bfi_k(t)$ intersect then $p(\bfr_\ell,\bfi_k\mid\bfz_{1:t})=0$; otherwise, $p(\bfr_\ell,\bfi_k\mid\bfz_{1:t}) = p(\bfr_\ell\mid\bfzr)p(\bfi_k\mid\bfzeye)$. We let $\bdi$ be the point-wise operator acting on $\probot,\peye$.  
\begin{lemma}[\textbf{$\interactioneye$ must have $\{0\}$-support}]
\label{lemma:bdi-unique}
If $\probot,\peye$ are GP mixtures and if the joint decomposes as in Equation~\ref{eq:decomposeeye}, then $\interactioneye$ must have $\{0\}$-support. 
\end{lemma}
\begin{proof}
Let $\bfi_k\sim\peye$, $\bfr_\ell\sim\probot$, $\eta^{\bfi}_k = p(\bfi=\bfi_k\mid\bfzeye)$, and $\eta^R_{\ell} = p(\bfr=\bfr_{\ell}\mid\bfz^R_{1:t})$ throughout the proof. 
\begin{figure}[tp]
\vspace{-0pt} 
\centering 
\includegraphics[width=0.5\textwidth]{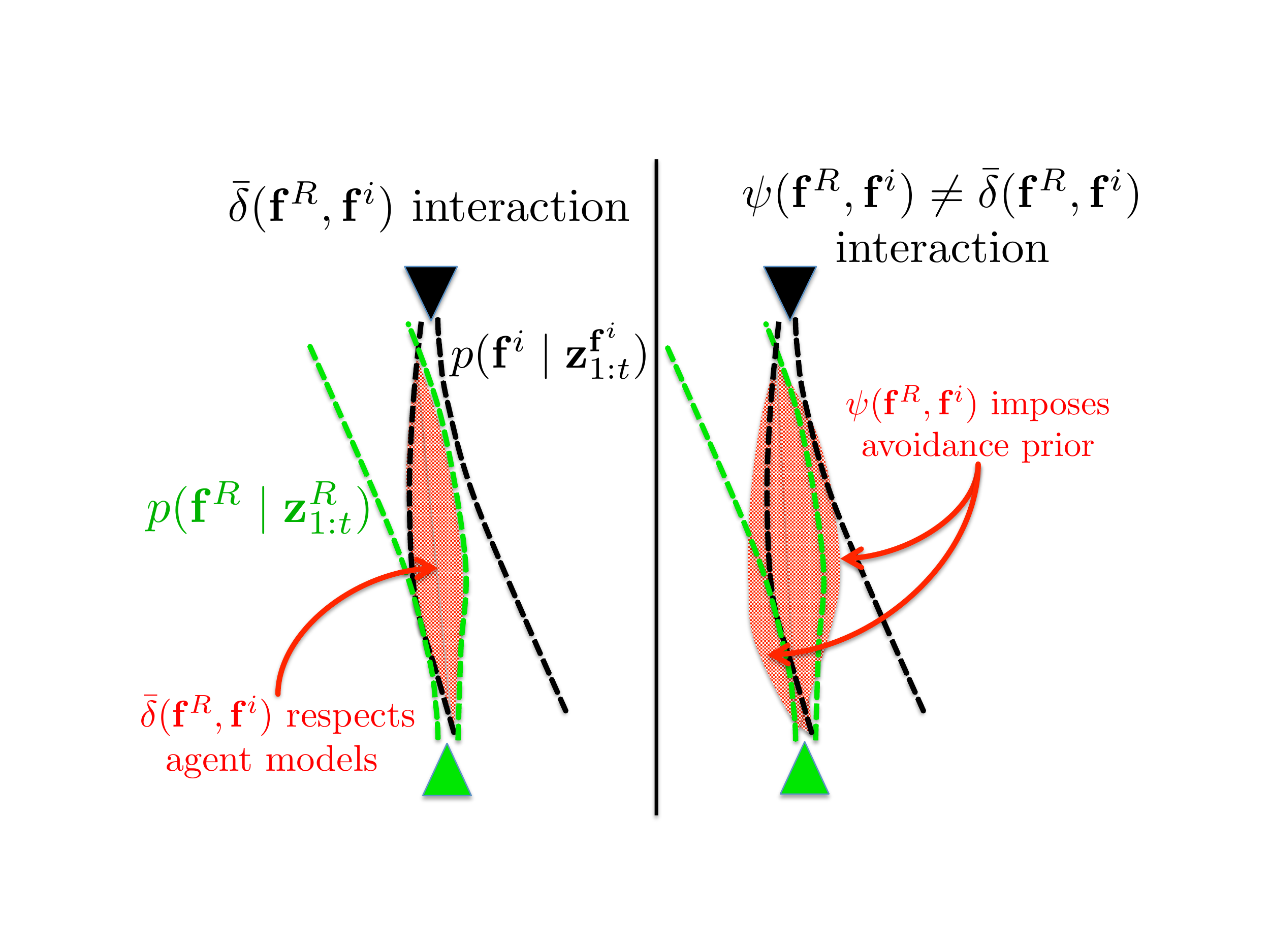}
\vspace{-5pt} 
\caption{Human in black, robot  in green, dotted lines one standard deviation; $\bdi$ only zeros overlap of $\peye,\probot$; $\interactioneye \neq \bdi$ \emph{imposes avoidance prior} on $\peye$ and $\probot$.}
\label{fig:interaction-bar-delta} 
\vspace{-15pt}
\end{figure} 
Let $\interactioneye=\bdieye$; then, for all $\ell$, $p(\bfr_\ell,\bfi_k\mid\bfz_{1:t})=\bar\delta(\bfr_\ell,\bfi_k)\eta^R_\ell\eta^{\bfi}_k.$  If $\nexists t\in[1,T]$ such that $\bfr_{\ell}(t)=\bfi_k(t),$ then $\bar\delta(\bfr_\ell,\bfi_k)\eta^R_\ell\eta^{\bfi}_k=\eta^R_{\ell}\eta^{\bfi}_k$; otherwise, it is zero.
Thus, $\bdieye$ respects the agent flexibility data contained in $\probot$ while preventing collision trajectories.  The same argument can be made for $\peye$; thus, $\bdieye$ respects the flexibility data in $\probot$ and $\peye$.

Conversely, let $\interactioneye\neq\bdieye$.  Then $p(\bfr_{\ell},\bfi_k\mid\bfz_{1:t}) = \psi(\bfr_{\ell},\bfi_k)\eta^R_{\ell}\eta^{\bfi}_k$.   Since $\interaction$ has finite support, the flexibility of $\probot$ and $\peye$ will be misrepresented: $\interaction$ is as an flexibility prior, but flexibility data is contained in the agent distributions.  
\end{proof}
\noindent Other $\{0\}$-support functions satisfy these requirements.  However, $\bdieye$ has the strongest collision avoidance properties.  In~\cite{trautmaniros}, $\interactioneye$ was based on observations of human cooperation; such a static and generic interaction function is statistically inappropriate and \emph{unnecessary}.  By capturing intention and flexibility in $\probot$ and $\peye$, statistical correctness demands $\interactioneye=\bdieye$.  

\subsection{Implementing $\bdieye$}
\label{sec:implement-bar-delta}
\noindent Unfortunately, $\bdieye\probot\peye$ is not analytic.  To construct an approximation, we begin by defining
\begin{align*}
\dieye\equiv
\begin{cases}
1 &\text{if $\nexists t\in[1,T]$ such that  } \bfr(t)\neq\bfi(t)\\
0 &\text{if $\exists t\in[1,T]$ such that } \bfr(t) \neq \bfi(t),
\end{cases}
\end{align*}
and note that $\bdieye = 1 - \dieye.$  Thus,
\begin{align}
\label{eq:bdi-action}
p(\bfr,&\bfi\mid \bfzt) =\bdieye\probot\peye \nonumber\\
&= (1 - \dieye)\Big[\robotsumshort\eyesumshort  \Big]\nonumber\\
&= \sumr\sumeye (1 - \dieye)\wrobot\weye \normalrobotshort\normaleyeshort.
\end{align}
Recalling $\pnot$ (Definition~\ref{defi:pnot}), we define single agent $\sigp.$
\begin{defi}
\label{defi:sigp-single}
\textbf{Single agent $\sigp$} is defined as
\begin{align}
\label{eq:sigp}
&\pigpeyez = \pnot \Big[\robotsumshort\eyesumshort\Big] \nonumber\\
&=\sumr\sumeye \znot \wrobot\weye\normalrobotshort \normaleyeshort\\
&= \sumr\sumeye \Lambdalk^{R,\bfi}\wrobot\weye \normalrobotshort \normaleyeshort \nonumber,
\end{align}
where $\wrobot,\weye$ are defined after Equation~\ref{eq:gp-mixtures}.   
\end{defi}

$\Lambdalk^{R,\bfi}=\znot$ measures overlap (regions of intersecting trajectories) between $\normalrobotshort,\normaleyeshort$.  Thus, $\Lambdalk^{R,\bfi}$ gives exponentially more weight to basis pairs $\normalrobotshort \normaleyeshort$ with less overlap than to those with more overlap.  The effect of $\pnot$ is to give the most weight to bases that balance collision avoidance against human and robot flexibility and intent $\wrobot\weye$; probability mass of $\pigpeyez$ is thus shifted around collision regions (right, Figure~\ref{fig:approximate-bar-delta}).  Similarly, in Equation~\ref{eq:bdi-action}, $[1-\dieye]$ zeros out intersecting trajectories in $\normalrobotshort,\normaleyeshort$.  The effect is to shift probability mass around collision regions while conserving $\wrobot\weye$ (left, Figure~\ref{fig:approximate-bar-delta}).

Equation~\ref{eq:sigp} is motivated by Lemma~\ref{lemma:bdi-unique}.  However, $\sigp$ is appealing in its own right: since $\Lambdalk^{R,\bfi}$ falls off exponentially as GPs overlap, \textbf{sparsity} is inherent to $\sigp$.  Further, bases with large values of $\Lambdalk^{R,\bfi}\wrobot\weye$ \textbf{simultaneously optimize} joint collision avoidance, intention and flexibility. 

\emph{Sparsity is thus a natural precursor to optimality.}
\subsection{Multiple agent formulation}
\noindent We extend Definition~\ref{defi:sigp-single} to $\nt$ interacting agents:
\begin{align*}
\pigpshortz &= \interaction\probot\pf \\
&=\interaction\probot\prod_{i=1}^{\nt}\peye \\
&=\interaction \robotsumshort \prod_{i=1}^{\nt}\eyesumshort.
\end{align*}

\noindent Lemma~\ref{lemma:bdi-unique} applies to $\interaction$; if  we let $\bdi = \prod_{i=1}^{\nt}\bdieye\prod_{j> i}^{\nt}\bar\delta(\bff^i,\bff^j)$ then $\sigp$ takes the form
\begin{align*}
&\pigpshortz  = \bdi \probot\pf\\
&=\prod_{i=1}^{\nt}\bdieye\prod_{j> i}^{\nt}\bar\delta(\bff^i,\bff^j)\Big[\probot \prod_{i=1}^{\nt}\peye\Big]\\
&=\prod_{i=1}^{\nt}\bdieye\prod_{j> i}^{\nt}\bar\delta(\bff^i,\bff^j) \Big[\robotsumshort\prod_{i=1}^{\nt}\eyesumshort\Big].
\end{align*}

\begin{figure}[tp]
\vspace{5pt} 
\centering 
\includegraphics[width=0.5\textwidth]{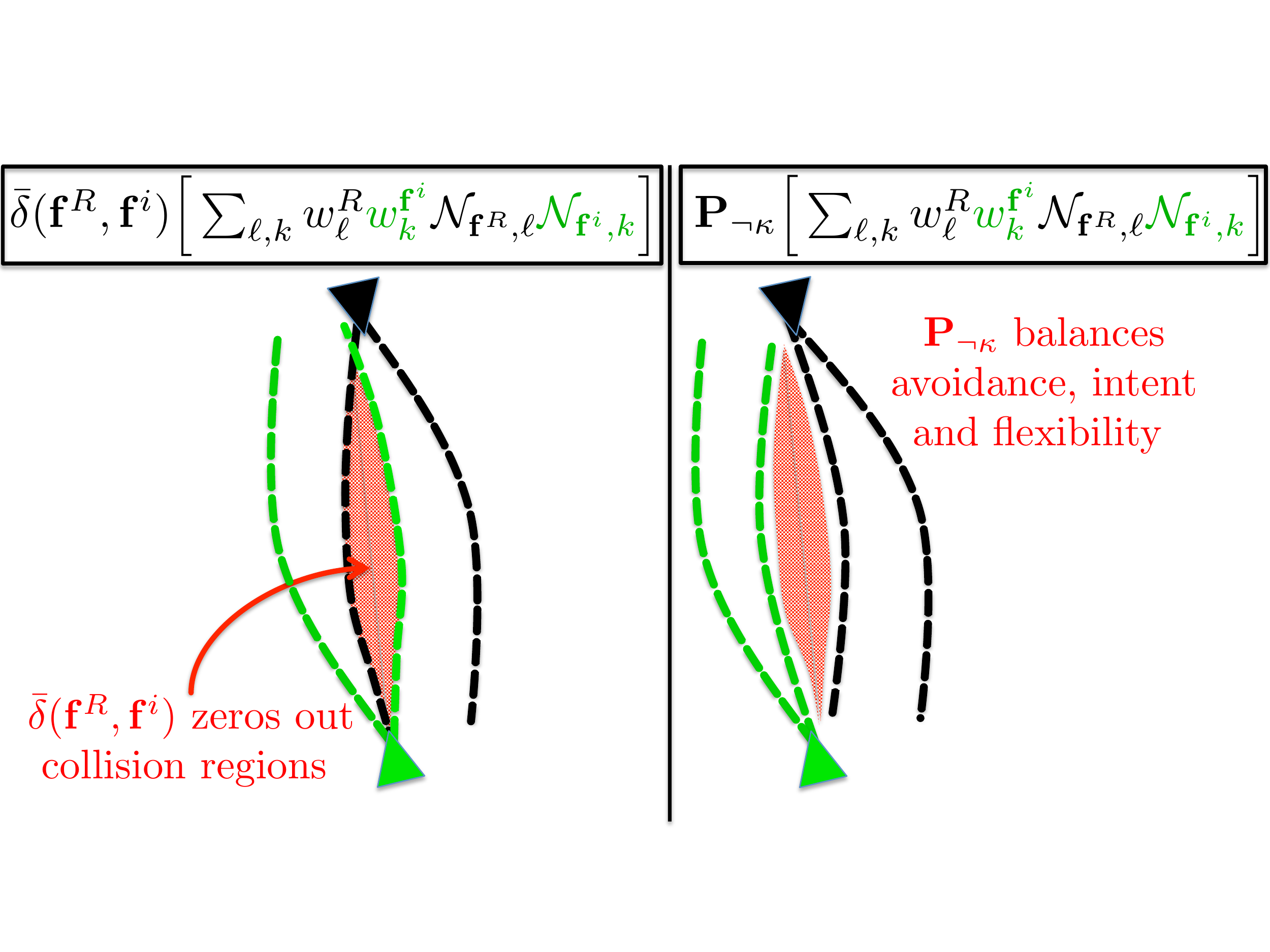}
\vspace{-5pt} 
\caption{\textbf{L:} $\bdieye$ zeros out overlap.  \textbf{R:} after $\pnot$ operates, GPs avoid collision region via $\znot$ while preserving intent via $\wrobot,\weye$.}
\label{fig:approximate-bar-delta} 
\vspace{-10pt}
\end{figure} 
\noindent\textbf{Example} Let $\nt=2$ and $\nr=N^{\bff^1}_t=N^{\bff^2}_t=1$. Then
\begin{align}
\label{eq:ma-example}
&\pigpshortz =\prod_{i=1}^{2}\bdieye\prod_{j> i}^{2}\bar\delta(\bff^i,\bff^j) \calN_{\bfr,1}\calN_{\bff^1,1}\calN_{\bff^2,1} \nonumber\\
&=\bar\delta(\bfr,\bff^1)\bar\delta(\bfr,\bff^2)\bar\delta(\bff^1,\bff^2) \calN_{\bfr,1}\calN_{\bff^1,1}\calN_{\bff^2,1},
\end{align}
The effect of $\bdi$ on $\probot\pf$ is to minimize the probability mass in regions of overlap between the GPs.  To construct $\sigp$ we follow the factorization of $\bdi$.  We note that $\pnot$ acts pairwise on GP bases:
\begin{align*}
&\pigpshortz = \pnot \Big[\robotsumshort\prod_{i=1}^{\nt}\sumfi\weyei\normaleyeshorti\Big]. 
\end{align*}
We explain the expression inside the brackets: $\pigpshortz$ is a mixture of weighted bases $\wrobot w^{\bff^1}_{k_1}\cdots w^{\bff^{\nt}}_{k_{\nt}}\normalrobotshort \calN_{\bff^1,k_1}\cdots \calN_{\bff^{\nt},k_{\nt}}$ with $\ell\in \{1,\ldots, \nr\}$ and for each $i=1:\nt$ we have $k_i \in \{1,\ldots, \neye \}$.  Thus, before $\pnot$ operates on this mixture, there are $N= \nr\prod_{i=1}^{\nt}\neye$ components.  We generalize $\pnot$ to act on each pair $\ell,k_i$ of each base $\normalrobotshort \calN_{\bff^1,k_1}\cdots \calN_{\bff^{\nt},k_{\nt}}$ (we leverage the empirical observation from~\cite{trautmanicra2013} that pedestrian robot interaction is more important than pedestrian-pedestrian interaction).  Thus
\begin{align*}
\pnot&\Big[\normalrobotshort \calN_{\bff^1,k_1}\cdots \calN_{\bff^{\nt},k_{\nt}} \Big] = \\
&\bLambda^{R,\bff^1}_{\ell,k_1}\cdots \bLambda^{R,\bff^{\nt}}_{\ell,k_{\nt}}\normalrobotshort \calN_{\bff^1,k_1}\cdots \calN_{\bff^{\nt},k_{\nt}}.
\end{align*} 
\begin{defi} Multi-agent $\sigp$ is defined as
\label{defi:sigp-n}
\begin{align}
\label{eq:sigp-n}
&\pigpshortz = \nonumber\\
& \sum_{\eta=1}^N[\bLambda^{R,\bff^1}\cdots \bLambda^{R,\bff^{\nt}} w^R w^{\bff^1}\cdots w^{\bff^{\nt}}]_{\eta}\big[\calN_{\bfr} \calN_{\bff^1}\cdots \calN_{\bff^{\nt}}\big]_{\eta}\nonumber\\
&= \sum_{\eta=1}^N [\bLambda \bs{w}]_\eta [\calN_{\bfr} \calN_{\bff^1}\cdots \calN_{\bff^{\nt}}\big]_{\eta}
\end{align}
where $\eta \in \{1,\ldots, \nr\prod_{i=1}^{\nt}\neye \}$.  That is, $\eta$ enumerates all the possible combinations of bases discussed above.  Let $[\bLambda \bs{w}]_\eta=[\bLambda^{R,\bff^1}\cdots \bLambda^{R,\bff^{\nt}} w^R w^{\bff^1}\cdots w^{\bff^{\nt}}]_{\eta}$.  From Equation~\ref{eq:igp} the optimal action is $u^I(t+1) = \bff^{R^*}(t+1).$
\end{defi}

\subsection{Computational complexity} 
\label{sec:complexity}
\noindent Consider the case of a single agent and a robot (Figure~\ref{fig:approximate-bar-delta}).  Here $\Lambdalk^{R,\bfi}$ has non-trivial value except for a left and right mode for each agent (since $\Lambdalk^{R,\bfi}$ decays exponentially as the GPs overlap).  Thus, even though $N = \nr\neye$, the number of non-trivial bases is $\nsigp = 2.$ For larger values of $\nt$, this result still holds: since our GP basis models ``high level'' activity---each agent must maintain a ``left'' and ``right'' GP basis---$\nsigp=\mathcal O(2)$.  Of course, our basis elements $\calN_{\bff^1}\cdots \calN_{\bff^{\nt}}$ grow more complex with $\nt$---a ``left'' and a ``right'' GP for each agent---so the number of GPs we compute is $\mathcal O(2^{\nt})$.  
\begin{lemma}[\textbf{$\sigp$ $\mathcal O(2^{\nt})$-sparse}]
\label{lemma:sparsity}
To approximate Equation~\ref{eq:sigp-n} accurately, $\mathcal O(2)$ GPs for each of the $\nt$ agents is required.  We say that multi-agent $\sigp$ is $\mathcal O(2^{\nt})$-sparse.
\end{lemma}
\begin{proof}
By inspection of Equation~\ref{eq:sigp-n}, we see that each $\big[\bLambda^{R,\bff^1}\cdots \bLambda^{R,\bff^{\nt}}\big]_{\eta}$ term only has non-trivial weight for GPs that do not overlap.  By the argument above, this is $\mathcal O(2)$ GPs for each of the $\nt$ agents.  
\end{proof}

GPs, as used here, are fast to compute.   In~\cite{trautmanicra2013}, 60 GPs were computed, and then 10000 samples were drawn from each GP, all in about 0.1 seconds.  Critically, however, we do not have an analytic transform into sparse space; instead, we resort to approximate inference to compute the high value bases.     

In contrast, trajectory sampling requires a decision at both the high level (``left'' or ``right'') and at each time step in the look-ahead ($T=30$ in~\cite{trautmanicra2013}).  In~\cite{trautmanicra2013}, $\nt=5$ (although up to 30 pedestrians were within a few meters of the robot); using this value, the number of samples needed for accurate navigation in dense crowds is $\mathcal O(2^{2\nt T})=\mathcal O(10^{90})$.  
\begin{lemma}
SBMPs (Def.~\ref{defi:sbmp}) need $\mathcal O(2^{2\nt T})$ computations. 
\end{lemma}
\begin{proof}
Based on the argument above, sampling based motion planners need $\mathcal O(2^{2\nt T})$ to provide coverage of the space.  In general, coverage is needed---especially for low probability events (Figure~\ref{fig:vd-fail})---to ensure safety failures do not occur.
\end{proof}
\noindent \emph{Note the difference here with $\sigp$:} by exploiting the kernel trick of GPs, the mean and covariance function are used as proxies for families of trajectories.  This is why the computational complexity of $\sigp$ is $\mathcal O(2^{2\nt})$ and not $\mathcal O(2^{2\nt T})$.   Further, by \emph{interacting} the GPs, and dynamically creating free space by co-evolving human and robot distributions, we efficiently guide probability mass placement (Figure~\ref{fig:vd-fail}).  As~\cite{kuderer-ijrr-2016} notes, the combinatorics of multiple interacting agents is too extreme to hope for tractability---and thus optimality guarantees---without exploiting the structure of the space.

\begin{figure}[t]
\vspace{5pt} 
\centering 
\includegraphics[width=0.46\textwidth]{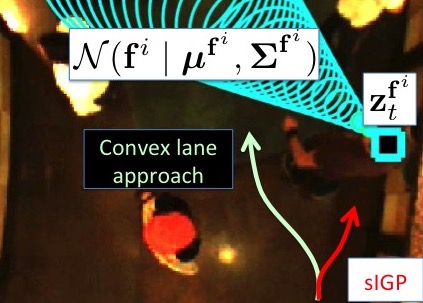}
\vspace{-5pt} 
\caption{Convex lane (CL) and independent agent methods fragile to safety failures.  CL derives free space based on $\bfz^{\bfi}_t$, so navigation is in the person's path.  $\sigp$ derives open lanes by interacting robot and human GPs.  Thus $\sigp$ knows to follow behind the person's current position.  Data from~\cite{trautmanicra2013}}
\label{fig:vd-fail}
\vspace{-15pt} 
\end{figure}

\subsection{Approximate inference of the GP basis of $\pigpshortz$}
\label{sec:inference}
\noindent Inference of $\sigp$ is slightly different than conventional inference of distributions.  In particular, if we find the \emph{basis element} $ \calN_{\bfr,\ell}\prod_{i=1}^{\nt} \calN_{\bfi,k_i}$ with the largest coefficient $\balpha = \prod_{i=1}^{\nt}\bs{\Lambda}^{R,\bfi}_{\ell,k_i}\prod_{j>i}^{\nt}\bs{\Lambda}^{\bfi,\bff^j}_{k_i,k_j}\wrobot\weyei$ then we have found the optimal navigation strategy.  Unfortunately, we do not have an analytic procedure to discover basis elements with large coefficients $\balpha$; we thus resort to approximation.  Equation~\ref{eq:gp-mixtures} did not specify how to generate the GP bases.  Previous work has addressed this: in~\cite{trautmanicra2013}, goals were inferred and GPs were conditioned on those goals.  In~\cite{joseph-dps-over-gps}, the number of components of a GP mixture was learned with Dirichlet process priors.

Instead, we sample GPs directly.  This is tractable since GPs are specified by the mean and covariance function.  Thus,
\begin{align*}
\bmurl &\sim p(\bmur\mid\bfzr)\\
\bSigmarl&\sim p(\bSigmar\mid\bfzr)
\end{align*}
and similarly for $\bmufiki,\bSigmafiki$.  We choose $p(\bmur\mid\bfzr) = \calN(\bfr \mid \bmur_0,\bSigmar_0)$, the most likely GP given the data $\bfzr$, and sample $\bSigmarl \sim \alpha\bSigma_0$ where $\alpha \sim \mathcal U[.1, 1]$ (and similarly for each agent $\bfi$).  The weights are computed as the likelihood of the data $\bfz_{1:t}$ conditioned on the samples: $\wrobot=\calN(\bfr=\bfzr \mid \bmur_\ell,\bSigma_\ell)$ and $\weye= \calN(\bfi=\bfzeye \mid \bmufi_k,\bSigmafi_k).$

\section{Optimality theorems of Crowd Navigation}
\label{sec:sparsity}
\begin{theorem}[\textbf{$\sigp$ optimal}]
\label{thrm:sigp-optimal}
sIGP is jointly optimal with respect to collision avoidance, intent and flexibility of the robot and the $\nt$ agents.
\end{theorem}
\begin{proof}
Inspecting Equation~\ref{eq:sigp-n}, we see that the  coefficient $\eta^* = \max_{\eta}[\bLambda\bs{w}]_{\eta}$ is an optimal balance of collision avoidance, intent and flexibility of each agent and the robot.  Since the collision avoidance coefficients constrain a single GP to a convex lane, we have that $\argmax_{\bfr,\bff}\pigpshortz = [\calN_{\bfr} \calN_{\bff^1}\cdots \calN_{\bff^{\nt}}\big]_{\eta^*}$ and $u^I(t+1) = \bmur_{\eta^*}$---see Eq.~\ref{eq:igp}.
\end{proof}

\noindent These sparsity (Lemma~\ref{lemma:sparsity}) and optimality results make sIGP a compelling crowd navigation approach.  We consider sparsity and optimality properties of state of the art approaches. 
\begin{figure}[t]
\vspace{0pt} 
\centering 
\includegraphics[width=0.44\textwidth]{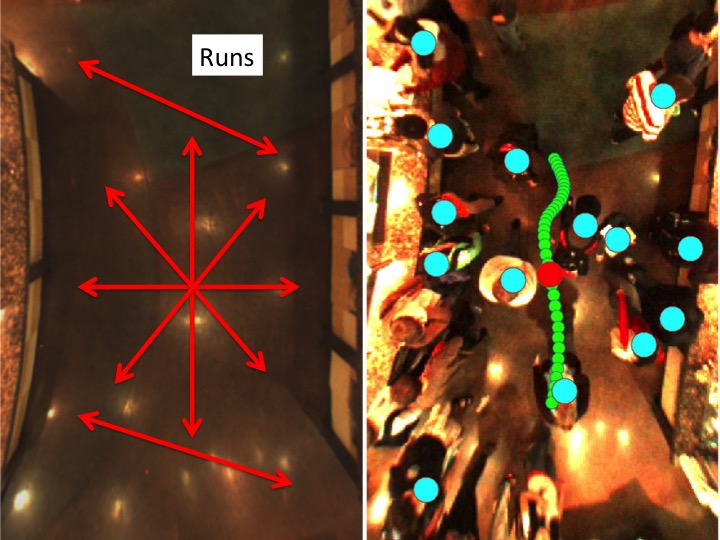}
\vspace{-0pt} 
\caption{\textbf{L:}  12 Runs were executed.  Double arrow indicates planner went in both directions.  \textbf{R:} Snapshot from run.  Planner history and prediction in green, current position in red.  Cyan dots are $\nt=14$ agents. Data from~\cite{trautmanicra2013}}
\label{fig:experiment-pic} 
\vspace{-20pt}
\end{figure}
\begin{coro}
\label{coro:predict-then-act}
\textbf{Independent agent based planning methods} (Definition~\ref{defi:decoupled-planning}) are suboptimal with respect to collision avoidance and intent and flexibility of the robot and human agents.
\end{coro}
\begin{proof}
Recall that independent planning methods assume $C(\bfr,\bff) = C_R(\bfr)C_{\bff}(\bff)$.  But the discussion of $\balpha_\eta = [\bLambda^{R,\bff^1}\cdots \bLambda^{R,\bff^{\nt}} w^R w^{\bff^1}\cdots w^{\bff^{\nt}}]_{\eta}$ in Theorem~\ref{thrm:sigp-optimal} showed how the coupling term $\balpha_{\eta}$ balances collision avoidance against intent and flexibility.  If $C(\bfr,\bff) = C_R(\bfr)C_{\bff}(\bff)$ then these competing objectives cannot be balanced.  See Figure~\ref{fig:vd-fail}.
\end{proof}

The practical ramification of independent modeling is efficiency suboptimal (overcautious) or safety suboptimal (overaggressive) behavior.  This suboptimality was described in~\cite{trautmaniros} and empirically observed in~\cite{trautmanicra2013} (the \emph{freezing robot problem}).  

By assuming that the human is independent of the robot (\emph{predict-then-act}), the \emph{predicted} collision probability is much larger than the \emph{true} collision probability. Without damping the cost function, this leads to overcautious behavior.  Alternatively, damping the cost function leads to robot-human miscalibration: in congestion, the only way that a robot can move safely and efficiently is by leveraging human cooperation~\cite{trautmaniros}.  

\begin{lemma} 
The \textbf{\emph{convex lane approach}} of Definition~\ref{defi:convex-lane} is a special case of $\sigp$.
\end{lemma}
\begin{proof}
If we restrict convex lane identification to time $t$---that is, we restrict the means and covariances to be evaluated at $\bmurl(t), \bmufiki(t)$ and $\bSigmarl(t), \bSigmafiki(t)$ in $\bs{\Lambda}^{R,\bfi}_{\ell,k_i}$---then we recover the convex lane approach.  Thus, in Equation~\ref{eq:sigp-n}, if we choose $\bs{\Lambda}^{R,\bfi}_{\ell,k_i} = \bs{\Lambda}^{R,\bfi}_{\ell,k_i}(t)$ the convex lane approach is recovered.
\end{proof}

\begin{coro}
Convex lane approaches are suboptimal with respect to collision avoidance, intent and flexibility of the robot and human agents.
\end{coro}
\begin{proof}
Because convex lane approaches consider pedestrians as static, collision avoidance errors occur (Figure~\ref{fig:vd-fail}); the method is thus multi-objective suboptimal.
\end{proof}

By Corollary~\ref{coro:predict-then-act}, independence assumptions between the human and robot in planning leads to suboptimality.  But the convex lane approach assumes that the human, at time $t$, is not responding to the robot's movements at the ``lane'' level.  Thus, convex lane approaches assume a human-robot independence.

\begin{lemma} $\sigp$ is performance lower bounded by the \emph{convex lane approach} of Definition~\ref{defi:convex-lane}.
\end{lemma}
\begin{proof}
Suppose that $\bfi$ is static.  Then $\normaleyeshort$ is centered at $\bfz^{\bfi}_t,$ $\bs{\Lambda}^{R,\bfi}_{\ell,k_i}= \bs{\Lambda}^{R,\bfi}_{\ell,k_i}(t)$ and the convex lane approach is recovered.

Suppose that $\bfi$ is moving.  Then $\sigp$ optimizes collision avoidance, intent and flexibility; Figure~\ref{fig:vd-fail} illustrates how the convex lane approach can fail while $\sigp$ is optimal.
\end{proof}

\section{Evaluation}
\begin{table}[t]
\label{table:findings}
\begin{center}
 \begin{tabular}{|c c c c c|} 
 \hline
Run & Safety (m) & Speed (m/s) & Run-time (s) & $\#$ Samples  \\ [0.5ex] 
 \hline
 1 & 0.22 & 1.2 & 0.451& 500  \\ 
 \hline
2 & 0.11 &1.1 & 0.47&500\\
 \hline
 3 & 0.24 & 1.1& 0.42 &500\\
 \hline
 4 & 0.32 & 1.3 & 0.45&500\\
 \hline
 5 & 0.03 & 1.0 & 0.4 &500\\
 \hline
 6 & 0.07 & 1.0 & 0.46&500\\
 \hline
 7 & 0.3 & 0.9 & 0.45&500\\
 \hline
 8 & 0.12 & 1.2 & 0.44&500\\
 \hline
 9 & 0.1 & 1.1 & 0.46&500\\
 \hline
 10 & 0.2 & 1.0 & 0.41&500\\
 \hline
 11 & 0.15 & 1.3 & 0.43&500\\
 \hline
 12 & 0.11 & 0.8 & 0.45&500\\ 
 \hline
 Average & 0.16m & 1.08m/s & 0.44s&500\\
 \hline
\end{tabular}
\caption{Experimental results  with $\nt=14$ with 8 interactive pedestrians. Safety is closest distance robot got to a pedestrian; speed is the average velocity of the run; run-time is average computation time; $\#$ samples is number of GP samples. Note: $\mathcal O(2^{\nt^{active}})=\mathcal O(256)=500.$} 
\end{center}
\vspace{-20pt}
\end{table}
\noindent Following the approach described in Section~\ref{sec:inference}, we empirically examine the $\sigp$ approach, using data from~\cite{trautmanicra2013}.  In particular, we examine the computation, safety, robot speed and number of samples required for the scenario in Figure~\ref{fig:experiment-pic}, and present our results in Table I.  In this scenario, $\nt=14$ pedestrians are present, and all are computed over.  Notably, about 6 of the pedestrians are near the wall and do not leave their position during the run; they serve primarily as a check on $\sigp$'s ability to compute over large numbers of agents and manage the uncertainty explosion with large numbers of agents.  Critically, then, about 8 people interacted with the robot, so $\nt^{active} \approx 2^8$; thus complexity was $\mathcal O(2^{\nt^{active}}) = \mathcal O(256).$  Since 500 samples were used, this provides empirical validation of Lemma~\ref{lemma:sparsity}.  Further, the left pane of Figure~\ref{fig:experiment-pic} shows that the robot is minimally disruptive while avoiding collisions, providing empirical validation for Theorem~\ref{thrm:sigp-optimal}.
We conducted runs in 12 directions (left pane, Figure~\ref{fig:experiment-pic}).  The right pane of Figure~\ref{fig:experiment-pic} provides a snapshot of $\sigp$ in mid run.  Note how $\sigp$ is able to weave smoothly through highly dense traffic with real time computation.  Performance for the other directions showed similar characteristics.  Average human walking speed is 1.3 m/s, and people often brushed shoulders.

\bibliographystyle{abbrv}
{\footnotesize
\bibliography{sparse-igp}
}

\end{document}